\pdfoutput=1
\documentclass[11pt]{article}

\usepackage[letterpaper,margin=1in]{geometry}
\usepackage[T1]{fontenc}
\usepackage{amsmath,amsthm,mathtools}
\IfFileExists{newtxtext.sty}{\usepackage{newtxtext,newtxmath}}{\usepackage{mathptmx}\usepackage{amssymb}}
\usepackage{microtype}
\usepackage{setspace}
\usepackage{booktabs,threeparttable,array,tabularx}
\usepackage{graphicx}
\usepackage{enumitem}
\usepackage{caption}
\usepackage{placeins}
\usepackage{natbib}
\usepackage{xcolor}
\usepackage[hidelinks]{hyperref}
\usepackage[nameinlink,noabbrev]{cleveref}

\graphicspath{{figures/}}
\setlist[itemize]{leftmargin=1.4em,itemsep=0.2em,topsep=0.3em}
\setlist[enumerate]{leftmargin=1.6em,itemsep=0.2em,topsep=0.3em}

\newtheorem{proposition}{Proposition}
\newtheorem{theorem}{Theorem}
\newtheorem{lemma}{Lemma}

\theoremstyle{definition}
\newtheorem{definition}{Definition}
\theoremstyle{remark}
\newtheorem{remark}{Remark}

\newcommand{\E}{\mathbb{E}}
\newcommand{\Prob}{\mathbb{P}}
\newcommand{\1}{\mathbf{1}}
\newcommand{\F}{\mathcal{F}}

\newcommand{\Loss}{\mathcal{L}}
\newcommand{\opt}{\mathrm{opt}}
\newcommand{\priv}{\mathrm{private}}
\newcommand{\cons}{\mathrm{consensus}}
\newcommand{\mkt}{\mathrm{market}}
\newcommand{\plan}{\mathrm{planner}}
\newcommand{\PoA}{\mathrm{PoA}}

\title{\textbf{The Shared Discovery Paradox}\\[0.35em]
\large How a One-Answer Rule Turns Better Information into Worse Search}
\author{Yohei Nakajima\thanks{Untapped Capital. Code, data, and an interactive guide: \texttt{github.com/yoheinakajima/shared-discovery-paradox}. All numerical claims in the paper are reproduced to stated precision by the accompanying verification script; see Appendix~D for the check list and claim-status taxonomy.}\\[2pt]
{\normalsize Untapped Capital}}
\date{July 2026}

\begin{document}
\maketitle
\vspace{-1.2em}

\begin{abstract}
Organizations often pool dispersed information into one ranking and then allow many agents to act on that shared view. In a discovery problem, this can improve beliefs while reducing coverage. We develop an exactly solvable benchmark with sixteen boxes, one target, eight searchers, and noisy private clues. Pooling raises the accuracy of the best single recommendation from $0.20$ to $0.3835$, but repeating that recommendation lowers group discovery from $0.8322$ under decentralized clue-following to $0.3835$. A coordinated eight-action portfolio using the same pooled reports reaches $0.8594$, and seven coordinated actions recover the decentralized benchmark. The paradox is not an information failure. It is a protocol failure: a one-answer rule compresses a portfolio of available actions into one repeated choice.

We then replace the planner with self-interested searchers who split a prize. The equal-split game is a potential game. Its anonymous symmetric equilibrium obeys a water-filling rule, and discovery equals $N$ times the common equilibrium payoff. In the canonical instance it achieves $0.5991$: strictly above consensus, but below both private search and the planner. The exact mixed price of anarchy is $2-1/N$, while uniform symmetric mixing yields the familiar $e/(e-1)$ limit. A sole-rescue reward, which pays only an agent who covers the target alone, makes every pure Nash equilibrium first-best. Finally, a latent common-cue model shows how correlated reports collapse effective discovery channels. The centralized planner gain rises strictly with copying, and in the canonical environment the symmetric market overtakes decentralized report-following at copying probability $c=0.788462$. In a proportional large-market limit the five-protocol ordering survives exactly: consensus discovery vanishes while blind, market, private, and portfolio search converge to $0.500$, $0.547$, $0.847$, and $0.874$. The contribution is a compact benchmark that separates information, allocation, incentives, and dependence into exact, reusable quantities.
\end{abstract}


\section{Introduction}

A group can become better informed and less likely to discover what it is looking for. The apparent contradiction arises when an institution solves an estimation problem and then uses the answer as if it had solved an allocation problem. Pooling information may identify the most likely target more accurately. Yet if many available actions are all assigned to that one target, the organization converts a portfolio of attempts into a repeated bet.

This paper isolates that mechanism in a deliberately small environment. Sixteen boxes contain one jackpot. Eight searchers receive noisy clues, each naming the correct box with probability $0.20$. If searchers independently follow their own clues, group discovery is $0.8322$. If they pool the clues, identify the posterior mode, and all follow it, the selected box is more accurate than any individual clue, but discovery falls to $0.3835$. If a coordinator instead assigns the same eight available actions to the eight highest-posterior boxes, discovery rises to $0.8594$. Even a coordinator who ignores the clues entirely and opens eight distinct boxes succeeds half the time: no information at all beats informed consensus. The same information supports both failure and improvement; the action protocol determines which one is realized.

Methodologically, the paper follows a familiar benchmark tradition. \citet{Akerlof1970} compresses quality uncertainty into a lemons market; \citet{Schelling1971} compresses local preferences into a checkerboard; and \citet{Arthur1994} compresses strategic crowding into the El Farol problem. Their power comes from making one mechanism visible in a toy environment that later work can extend, calibrate, or reject. Our narrower object is the separation between shared beliefs and diversified search. The comparison is methodological, not a claim of comparable scope. The aim is to make a recurring organizational failure measurable.

The paper develops four layers of the benchmark.

First, we separate the value of information from the quality of the action protocol. For an information set $\F$ and action budget $L$, the attainable frontier is the posterior mass of the best $L$-state portfolio. Protocol loss is the gap between that frontier and the set actually searched. In the canonical pooled-information environment, the entire $47.60$ percentage-point gap between consensus and the optimal eight-action portfolio is protocol loss. The resulting action-budget frontier yields an integer recovery budget: seven coordinated actions are sufficient to replace eight decentralized clue-followers.

Second, we ask whether decentralized incentives can supply the missing dispersion. After reports are pooled, $N$ self-interested searchers choose boxes and split a unit prize equally if several select the jackpot. This is a singleton congestion or covering game. It has an exact potential, and its anonymous symmetric equilibrium takes a water-filling form: posterior mass is traded off against expected prize dilution. The equilibrium improves on consensus at every posterior, but it can still collide excessively. Averaging over the canonical report distribution gives discovery $0.5991$ and only $2.67$ distinct searches in expectation. Thus the canonical ordering becomes
\[
G_{\cons}<G_{\mkt}^{\mathrm{sym}}<G_{\priv}<G_{\plan}.
\]
We sharpen the inefficiency result from a factor of two to the exact bound $2-1/N$ for all mixed Nash equilibria. This bound is known in the covering-game literature; the contribution here is a short search-specific proof and its integration into the Bayesian benchmark. Under a uniform posterior on $N$ boxes, anonymous symmetric mixing gives the familiar ratio converging to $e/(e-1)$, even though efficient asymmetric pure equilibria coexist.

Third, we identify a minimal incentive correction. Under a sole-rescue rule, the realized prize is paid only if exactly one searcher chose the target. The expected payoff to an agent is then exactly the posterior mass that the agent contributes uniquely to coverage. The resulting potential is social discovery itself. With at least as many boxes as searchers and positive posterior mass on every box, every pure Nash equilibrium assigns one searcher to each of the $N$ highest-posterior boxes. This is full pure-strategy implementation of the planner portfolio, not a dominant-strategy claim: equilibrium selection and mixed behavior remain separate issues.

Fourth, we introduce a latent common-cue copying model. Searchers may repeat the same unobserved source rather than draw independent clues. Copying reduces the number of effective discovery channels, can make raw report counts mis-rank posterior probabilities, and raises the value of reallocating duplicated capacity. We prove that the equal-budget centralized planner gain is strictly increasing in the copying probability. We then run the equal-split game on the herd-aware posterior. In the canonical environment, decentralized report-following initially dominates the symmetric market because independent clues create natural coverage. The ranking reverses at $c=0.788462$: beyond that point, channel collapse harms private report-following more than posterior concentration harms the dispersing market.

The ingredients have substantial precedents. Team theory, optimal search, parallel R\&D, organizational learning, information cascades, epistemic networks, potential games, and covering games each formalize part of the design problem \citep{MarschakRadner1972,Koopman1957,Nelson1961,March1991,Banerjee1992,BikhchandaniEtAl1992,Zollman2007,MondererShapley1996,Gairing2009}. We do not claim that communication-induced convergence, correlated signals, congestion-game inefficiency, or marginal-contribution incentives are individually new. The contribution is to place them in one exactly solvable Bayesian discovery benchmark and provide a common set of quantities: one-action value, portfolio value, protocol loss, recovery budget, effective channel count, anonymous-market value, price of anarchy, and rescue implementation.

\section{The sixteen-box benchmark}
\label{sec:benchmark}

There are $M=16$ boxes, exactly one of which contains a jackpot. The true box is $\theta\in\{1,\ldots,16\}$ and has a uniform prior. There are $N=8$ searchers. Conditional on $\theta$, searcher $i$ observes a clue $X_i$ satisfying
\[
\Prob(X_i=\theta\mid\theta)=p=0.20,
\qquad
\Prob(X_i=b\mid\theta)=q=\frac{1-p}{M-1}=\frac{0.8}{15}
\quad(b\neq\theta).
\]
The clues are conditionally independent in the baseline model. Each searcher can open one box, and the group succeeds if at least one opened box is the target.

\subsection{Three benchmark protocols}

\paragraph{Private clue-following.}
Each searcher opens the box named by the searcher's own clue. The group succeeds whenever at least one clue is correct, so
\begin{equation}
G_{\priv}=1-(1-p)^N=1-0.8^8=0.83222784.
\label{eq:private}
\end{equation}

\paragraph{Shared consensus.}
All reports are pooled and every action is assigned to the posterior mode. With a uniform prior and symmetric independent signals, posterior rank equals report-count rank, so consensus is plurality with uniform tie-breaking. Exact enumeration gives
\begin{equation}
G_{\cons}=0.38346871.
\label{eq:consensus}
\end{equation}
The recommendation is substantially better than one private clue, since $0.3835>0.20$. But eight identical actions cover one box, so recommendation accuracy and group discovery coincide.

\paragraph{Shared portfolio.}
A coordinator observes the same pooled reports and assigns actions to the $L$ highest-posterior boxes. With $L=8$,
\begin{equation}
G_{\plan}=0.85942125.
\label{eq:planner}
\end{equation}
The planner uses the information gain without compressing the action budget.

These values establish the basic paradox:
\[
0.20 < G_{\cons}=0.3835 < G_{\priv}=0.8322 < G_{\plan}=0.8594.
\]
Every participant can follow a better-informed recommendation while the group becomes less likely to discover the target.

\begin{definition}[Shared Discovery Paradox]
A Shared Discovery Paradox occurs when a shared-information protocol raises average action quality but lowers union success. If protocol $\Pi$ produces actions $A_1,\ldots,A_L$, define average action quality
\[
Q(\Pi)=\frac{1}{L}\sum_{\ell=1}^L \Prob(A_\ell=\theta)
\]
and discovery $G(\Pi)=\Prob(\theta\in\{A_1,\ldots,A_L\})$. The paradox is
\[
Q(\Pi^{S})>Q(\Pi^{P})
\quad\text{and}\quad
G(\Pi^{S})<G(\Pi^{P}).
\]
\end{definition}

In the canonical comparison, $Q(\Pi^{S})=0.3835>0.20=Q(\Pi^{P})$, while $G(\Pi^{S})=0.3835<0.8322=G(\Pi^{P})$.

Average action quality admits an exact accounting link to discovery. Let $K=\sum_{\ell=1}^{L}\1\{A_\ell=\theta\}$ be the number of actions that select the target. Then $LQ(\Pi)=\E[K]$ and $G(\Pi)=\Prob(K\ge1)$, so
\begin{equation}
LQ(\Pi)-G(\Pi)=\E\bigl[(K-1)^{+}\bigr].
\label{eq:redundancy}
\end{equation}
The right side is the expected number of redundant successful actions. Under consensus $K\in\{0,8\}$, so $8Q-G=7G=2.684281$: a correct recommendation produces eight identical hits. Under any collision-free portfolio at most one selected box can be the target, so the gap is zero. Private clue-following lies between the extremes, with $8Q-G=1.6-0.832228=0.767772$, because several clues can be simultaneously correct. The paradox is therefore a statement about duplication: a shared-information protocol can raise $Q$ precisely by concentrating $K$.

\paragraph{Blind-portfolio control.}
A blind control calibrates how expensive that duplication is. With a uniform prior and no clues at all, a coordinator who assigns the eight actions to eight distinct boxes succeeds with probability
\[
G_{\mathrm{blind}}=\frac{N}{M}=\frac{8}{16}=0.50.
\]
A completely uninformed but diversified portfolio therefore already outperforms informed consensus at $0.383469$. This is not a value-of-information comparison---the two protocols observe different things---but it makes the action-compression loss transparent: eight coordinated searches that ignore the clues entirely outperform eight informed searches that all repeat the pooled recommendation.\footnote{Even eight \emph{uncoordinated} blind guesses, drawn independently and uniformly, succeed with probability $1-(1-1/16)^{8}=0.403281>0.383469$.} We treat $0.50$ as a calibration line rather than a member of the institutional ordering, because blind search discards the available information.

\subsection{The full action-budget frontier}

Consensus evaluates pooled information at a one-action budget. Discovery is instead a budget-indexed frontier. Exact enumeration gives
\begin{center}
\begin{tabular}{c@{\qquad}c|c@{\qquad}c}
\toprule
Actions $L$ & Pooled discovery & Actions $L$ & Pooled discovery\\
\midrule
1 & 0.383469 & 5 & 0.735604\\
2 & 0.527532 & 6 & 0.794435\\
3 & 0.604528 & 7 & 0.836010\\
4 & 0.670581 & 8 & 0.859421\\
\bottomrule
\end{tabular}
\end{center}
The smallest coordinated budget that matches private clue-following is therefore
\[
L^*=7,
\qquad
G_{\mathrm{portfolio}}(7)=0.836010>0.832228=G_{\priv}.
\]

\begin{figure}[t]
\centering
\includegraphics[width=0.88\linewidth]{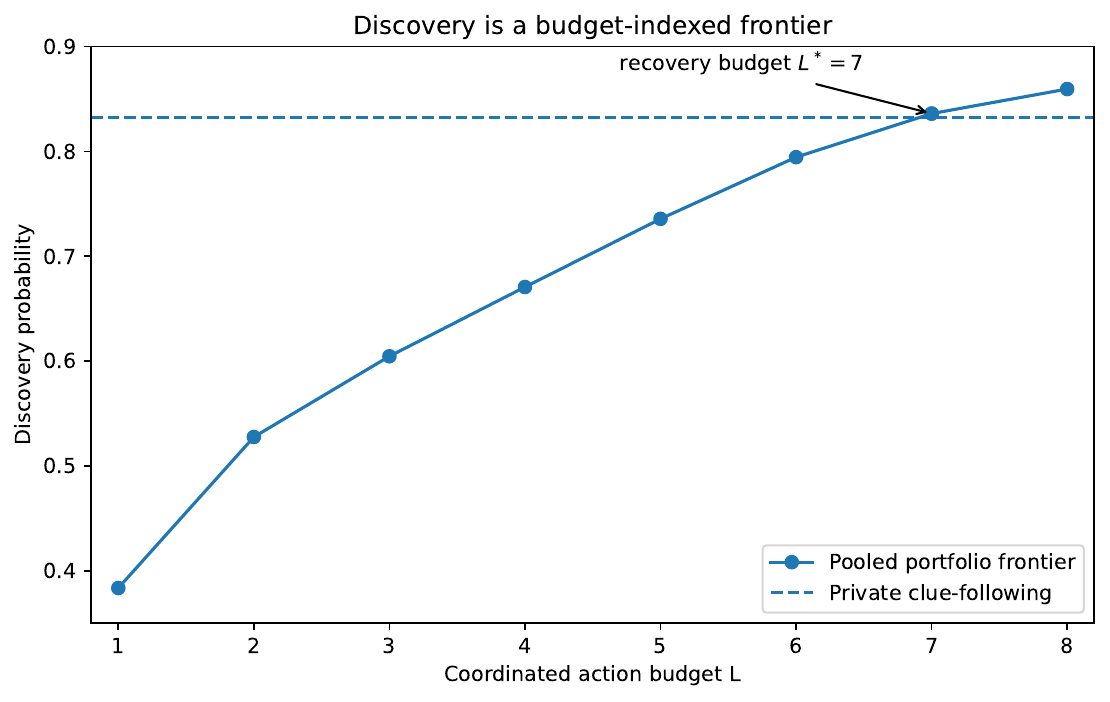}
\caption{The exact pooled-information frontier for the canonical benchmark. The one-action consensus value is only the left endpoint. Seven coordinated actions recover the performance of eight decentralized clue-followers.}
\label{fig:frontier}
\end{figure}

A second measure is the expected number of distinct actions. Private clue-following generates
\begin{equation}
D_{\priv}
=\bigl[1-(1-p)^N\bigr]+(M-1)\bigl[1-(1-q)^N\bigr]
=6.156850.
\label{eq:distinct-private}
\end{equation}
Consensus generates exactly one distinct action; the planner generates eight. These counts make the mechanism visible: pooling improves the ranking, while the protocol determines how much of the nominal action budget survives as coverage.

\section{Information value and protocol loss}
\label{sec:framework}

Let the target be $\theta\in\Omega$ with prior $\mu$. A decision architecture observes an information sigma-field $\F$. A protocol $\Pi$ maps its information into an action profile $(A_1,\ldots,A_L)$, where repeated actions are allowed. Let $S_\Pi=\{A_1,\ldots,A_L\}$ be the set of distinct searched states. Discovery is
\[
G_L(\Pi)=\Prob(\theta\in S_\Pi).
\]
For information $\F$ and budget $L$, define the attainable frontier
\begin{equation}
V_L(\F)=\E\left[\max_{S\subseteq\Omega,\ |S|\le L}
\Prob(\theta\in S\mid\F)\right].
\label{eq:frontier-general}
\end{equation}
In an atomic state space, the maximizing set contains the $L$ largest posterior probabilities. Define protocol loss
\begin{equation}
\Loss_L(\Pi;\F)=V_L(\F)-G_L(\Pi).
\label{eq:protocol-loss}
\end{equation}
Thus
\[
G_L(\Pi)=V_L(\F)-\Loss_L(\Pi;\F).
\]
This accounting identity separates the quality of available information from the efficiency with which an institution converts information and action capacity into coverage.

\begin{proposition}[Value of information]
\label{prop:voi}
If $\F_1\subseteq\F_2$, then $V_1(\F_2)\ge V_1(\F_1)$.
\end{proposition}

\begin{proof}
Every one-action rule measurable with respect to $\F_1$ remains feasible under $\F_2$. The optimal rule under the richer information cannot perform worse.
\end{proof}

\begin{proposition}[Fixed-budget optimality]
For every realization of $\F$, the posterior top-$L$ portfolio weakly dominates every protocol using at most $L$ distinct actions and the same realized information. Consequently, $V_L(\F)\ge G_L(\Pi)$.
\end{proposition}

\begin{proof}
The set generated by $\Pi$ is one feasible set in the maximization defining $V_L(\F)$. The posterior top-$L$ set maximizes conditional target mass. Take expectations.
\end{proof}

In the canonical pooled-information regime,
\[
V_8(\F^S)=0.859421,
\qquad
\Loss_8(\Pi_{\cons};\F^S)=0.859421-0.383469=0.475952.
\]
Conditional on the same pooled reports and the same nominal budget, the full gap between consensus and the planner is protocol loss. The comparison with private clue-following changes both the information architecture and the assignment authority, so it is a combined organizational benchmark rather than a unique causal decomposition.

\subsection{Independent symmetric signals}
\label{sec:signals}

Let $C_b=\sum_i\1\{X_i=b\}$ be the report count for box $b$. Under the uniform prior,
\begin{equation}
\Prob(\theta=b\mid C)
\propto p^{C_b}q^{N-C_b}
=q^N\left(\frac{p}{q}\right)^{C_b}.
\label{eq:posterior-count}
\end{equation}
Because $p>q$, posterior rank equals count rank. Consensus is plurality; the optimal pooled portfolio is the set of boxes with the largest counts, with randomization at cutoff ties.

\begin{proposition}[Private coverage dominates one-action consensus]
In the symmetric independent atomic model,
\[
G_{\cons}\le G_{\priv}=1-(1-p)^N.
\]
\end{proposition}

\begin{proof}
The plurality winner has at least one report. Therefore the event that consensus selects the target is a subset of the event that at least one private report names the target.
\end{proof}

The inequality holds throughout the model, but the strict Shared Discovery Paradox additionally requires the pooled one-action value to exceed $p$. With two searchers and uniform tie-breaking, pooled plurality accuracy is exactly $p$, so $N=2$ is a coverage-compression boundary case rather than a strict paradox.

\begin{remark}[Partial conformity is monotonically harmful]
\label{rem:conformity}
A natural intermediate institution lets only $a$ of the $N$ searchers defer to the pooled recommendation while the remaining $N-a$ follow their own clues. This dial only goes down. The plurality winner always carries at least one report, so it already belongs to the privately proposed set; an adopter therefore adds nothing to coverage and merely deletes their own proposal. Formally, for nested adopter sets the searched set satisfies $S_a\subseteq S_{a-1}$ realization by realization, so $G(a)$ is weakly decreasing from $G(0)=G_{\priv}$ to $G(N)=G_{\cons}$; which searchers adopt is immaterial by exchangeability. The decline is strict whenever the displaced clue is uniquely held, correct, and differs from the realized winner, an event of positive probability at every step. Exact enumeration at the canonical parameters gives
\[
G(a)=0.8322,\;0.7913,\;0.7461,\;0.6967,\;0.6429,\;0.5848,\;0.5222,\;0.4551,\;0.3835
\qquad(a=0,1,\ldots,8):
\]
even the first adopter costs $4.1$ percentage points. Dispersion must be informed (private clues), priced (the market of \cref{sec:selfish}), or assigned (the planner); it cannot be obtained by dosing conformity.
\end{remark}

\subsection{The proportional sparse-evidence limit}
\label{sec:scaling}

The sixteen-box environment is deliberately small, so it is natural to ask whether the ordering is a small-number artifact. Scale the environment proportionally: let the number of searchers grow with the number of boxes, $N_M/M\to\alpha\in(0,1)$, and hold the clue likelihood ratio $r=p_M/q_M>1$ fixed, so that
\[
p_M=\frac{r}{M-1+r},
\qquad
q_M=\frac{1}{M-1+r}.
\]
The canonical benchmark corresponds to $\alpha=1/2$ and $r=3.75$. Fixing $r$ is one natural invariant among several---one could instead preserve mutual information---but it keeps the report-count occupancy problem nondegenerate: the target receives an asymptotically $\mathrm{Poisson}(\alpha r)$ number of reports, while a typical false box receives approximately $\mathrm{Poisson}(\alpha)$.

\begin{proposition}[Proportional sparse-evidence limit]
\label{prop:scaling}
In the independent symmetric model with $N_M/M\to\alpha\in(0,1)$ and fixed likelihood ratio $r>1$, consider the blind coordinated portfolio ($N_M$ distinct boxes chosen without clues), pooled one-action consensus, private clue-following, and the pooled top-$N_M$ portfolio. As $M\to\infty$,
\[
G_{\mathrm{blind}}\longrightarrow\alpha,
\qquad
G_{\cons}\longrightarrow 0,
\qquad
G_{\priv}\longrightarrow 1-e^{-\alpha r},
\qquad
G_{\mathrm{portfolio}}\longrightarrow 1-(1-\alpha)\,e^{-\alpha(r-1)}.
\]
The anonymous symmetric equal-split market on the pooled posterior also converges:
\begin{equation}
G_{\mkt}^{\mathrm{sym}}\longrightarrow \alpha\Lambda,
\qquad
\text{where $\Lambda$ uniquely solves}
\quad
\sum_{j\,:\,\psi_j>\Lambda} f_j\,h^{-1}\!\bigl(\Lambda/\psi_j\bigr)=\alpha,
\label{eq:market-limit}
\end{equation}
with $f_j=e^{-\alpha}\alpha^{j}/j!$, $\psi_j=r^{j}e^{-\alpha(r-1)}$, and $h(x)=(1-e^{-x})/x$ strictly decreasing from $h(0)=1$. Writing $x_j=h^{-1}(\Lambda/\psi_j)$ on active count classes and $g_j=f_j\psi_j=e^{-\alpha r}(\alpha r)^{j}/j!$, the equivalent tilted form is $G_{\mkt}^{\mathrm{sym}}\to\sum_j g_j\bigl(1-e^{-x_j}\bigr)$.
\end{proposition}

The proof is in \cref{app:scaling}. For the canonical scaling $\alpha=1/2$, $r=3.75$, the solution of \cref{eq:market-limit} is $\Lambda=1.094021$, and the complete limit ordering is
\[
G_{\cons}\to 0
\;<\;
G_{\mathrm{blind}}\to 0.50
\;<\;
G_{\mkt}^{\mathrm{sym}}\to 0.547011
\;<\;
G_{\priv}\to 0.846645
\;<\;
G_{\mathrm{portfolio}}\to 0.873580.
\]
Because $(1-\alpha)e^{\alpha}<1$ on $(0,1)$, the portfolio limit strictly exceeds the private limit at every $\alpha$. The market's position strictly above consensus and weakly below the portfolio also holds at every $(\alpha,r)$, since the finite inequalities pass to the limit. The positions of the blind benchmark relative to private search and to the market, by contrast, depend on $(\alpha,r)$, and the displayed ordering is a property of the canonical scaling.

The market limit is water-filling against a Poisson count profile: in the limit the equilibrium abandons every box with fewer than two reports (canonically $\psi_1=0.948<\Lambda<\psi_2=3.555$) and spreads over the high-count classes in which the target, whose own count is $\mathrm{Poisson}(\alpha r)$, is hidden. The finite canonical value $0.599099$ sits above its limit $0.547011$, and the seeded Monte Carlo diagnostic in the verification script shows a monotone decline through $M\in\{16,64,256,1024\}$ ($0.5991$, $0.5740$, $0.5582$, $0.5494$): as the posterior spreads over more comparably ranked boxes, symmetric mixing wastes more of its budget on collisions inside the leading count classes.

Two readings of the canonical limit deserve note. First, information without assignment and assignment without information deliver nearly the same discovery: the market reaches $0.547$ holding the full pooled posterior but no coordination device, the blind portfolio reaches $0.500$ holding a coordination device but no information, and only their combination reaches $0.874$. Second, the market--private gap at scale admits an exact accounting. Since $G_{\priv}\to\Prob\bigl(\mathrm{Poisson}(\alpha r)\ge1\bigr)$ and the canonical limit market abandons the zero- and one-report classes,
\[
G_{\priv}-G_{\mkt}^{\mathrm{sym}}
\;\longrightarrow\;
g_1+\sum_{j\ge2}g_j\,e^{-x_j}
=0.287541+0.012094:
\]
essentially the entire gap is the probability that the target receives exactly one report. The singleton discoveries that private clue-followers keep for free are precisely what an anonymous market rationally abandons; collisions inside the classes the market does cover cost only $0.012$.

\begin{remark}[Abundant-evidence regime]
\label{rem:abundant}
The paradox is a sparse-evidence phenomenon. Hold $p\in(0,1)$ fixed with $N_M/M\to\alpha\in(0,1)$, so that the likelihood ratio $r_M=p/q_M\to\infty$. Then $G_{\cons}\to1$, and a fortiori every clue-using protocol converges to one: the paradox vanishes in levels, and distinctions among diversified protocols persist only in convergence rates. The clean dial is $\lambda_M=N_Mp_M$, the target's expected report count: \cref{prop:scaling} is the bounded-$\lambda$ regime, while whenever $\lambda_M/\log M\to\infty$ the union bound in \cref{app:scaling} gives consistency. The intermediate window $\lambda_M\asymp\log M$ is a multiple-comparisons threshold problem we leave open. The benchmark is aimed accordingly: consensus institutions are adequate where evidence about the leading candidate is abundant; the compression loss binds in sparse-evidence search---frontier science, early-stage investing, novel-hypothesis generation---where the target's support never separates from the noise.
\end{remark}

The limit sharpens the finite benchmark in two ways. First, the paradox is not a small-number artifact: consensus discovery vanishes while every diversified protocol retains a strictly positive limit. Second, large-scale consensus fails through an additional channel. The target's report count remains stochastically bounded, but the maximum count among linearly many false boxes diverges, so the plurality winner is eventually an extreme false positive: top-one selection loses a multiple-comparisons contest that the portfolio, which retains every named box, never enters. The pooled ranking still contains substantial information; converting it into one repeated action is what discards it. The portfolio formula is the asymptotic form of the finite reallocation identity that reappears under copying in \cref{sec:copying}: the planner's entire limiting gain over private search, $e^{-\alpha r}\bigl[1-(1-\alpha)e^{\alpha}\bigr]$, comes from redirecting duplicated capacity to unnamed states on the event that every clue misses.

\section{Selfish search and the price of coordination}
\label{sec:selfish}

The planner benchmark assumes authority to assign distinct actions. We now retain the common pooled posterior but remove the assignment rule. Fix a posterior
\[
\pi=(\pi_1,\ldots,\pi_M),
\qquad \sum_b\pi_b=1.
\]
There are $N$ risk-neutral searchers. Each simultaneously chooses one box. If the jackpot is in box $b$ and $k$ searchers chose $b$, each occupant receives $1/k$; all other payoffs are zero. The unit prize is therefore split equally among successful searchers.

\subsection{Payoffs, potential, and total welfare}

Suppose every other searcher uses the same mixed strategy $s=(s_1,\ldots,s_M)$. If a searcher chooses box $b$, the number of other occupants is $K\sim\mathrm{Binomial}(N-1,s_b)$. The reciprocal-share identity is
\begin{equation}
\E\left[\frac{1}{K+1}\right]
=\sum_{j=0}^{N-1}\binom{N-1}{j}s_b^j(1-s_b)^{N-1-j}\frac{1}{j+1}
=\frac{1-(1-s_b)^N}{Ns_b}.
\label{eq:reciprocal}
\end{equation}
Define
\begin{equation}
\phi(s)=\frac{1-(1-s)^N}{Ns},
\qquad \phi(0)=1.
\label{eq:phi}
\end{equation}
The expected payoff from choosing $b$ is
\begin{equation}
u_b(s)=\pi_b\phi(s_b).
\label{eq:payoff}
\end{equation}
The function $\phi$ is strictly decreasing from $1$ to $1/N$. Posterior value attracts searchers; congestion dilutes the prize.

For a pure action profile $a$, let $n_b(a)$ be the number of searchers on box $b$, and let $H_k=\sum_{j=1}^k1/j$ with $H_0=0$.

\begin{proposition}[Potential and payout identity]
The equal-split game is an exact potential game with potential
\begin{equation}
\Phi(a)=\sum_{b=1}^M \pi_b H_{n_b(a)}.
\label{eq:harmonic-potential}
\end{equation}
At every realized action profile, total searcher payoff equals the discovery indicator. Conditional on $\pi$, expected total payoff therefore equals discovery probability.
\end{proposition}

\begin{proof}
If one player moves from a box with occupancy $k$ to a box with occupancy $\ell$, the player's payoff changes by $\pi_c/(\ell+1)-\pi_b/k$. The potential changes by exactly the same amount. For the payout identity, when the target box is occupied, its occupants split a unit prize and total payoff is one; otherwise total payoff is zero.
\end{proof}

Pure Nash equilibria exist by the potential argument \citep{Rosenthal1973,MondererShapley1996}. They need not be unique or symmetric. The benchmark below focuses on the anonymous symmetric equilibrium selected when ex ante identical searchers use the same mixed rule. This selection is not a convenience: efficient asymmetric equilibria require role labels, conventions, or sequencing---exactly the coordination devices whose presence or absence the paper treats as the institutional variable. Anonymity is the absence of coordination devices, so the anonymous equilibrium measures what dispersion prices alone can supply; the planner and sole-rescue benchmarks then measure what restoring those devices adds.

\subsection{The water-filling equilibrium}

Order posterior masses as $\pi_{(1)}\ge\pi_{(2)}\ge\cdots$. The endpoint $\pi_{(1)}/N$ matters because it is the payoff when everyone chooses the top box.

\begin{proposition}[Anonymous symmetric equilibrium]
\label{prop:waterfill}
The game has a unique anonymous symmetric equilibrium strategy $s$.

\begin{enumerate}[label=(\roman*)]
\item If $\pi_{(1)}/N\ge\pi_{(2)}$, every searcher chooses the top box and the common payoff is $\lambda=\pi_{(1)}/N$.
\item Otherwise there is a unique $\lambda\in(\pi_{(1)}/N,\pi_{(1)})$ such that
\begin{equation}
s_b=
\begin{cases}
\phi^{-1}(\lambda/\pi_b), & \pi_b>\lambda,\\
0, & \pi_b\le\lambda,
\end{cases}
\qquad
\sum_b s_b=1.
\label{eq:waterfill}
\end{equation}
Every box in the support yields expected payoff $\lambda$; every excluded box yields at most $\lambda$.
\end{enumerate}
\end{proposition}

\begin{proof}
In a symmetric equilibrium, used boxes must yield the same payoff. Since $\phi$ is continuous and strictly decreasing, an active box with $\pi_b>\lambda$ has the unique probability in \cref{eq:waterfill}; a box with $\pi_b\le\lambda$ cannot attain payoff $\lambda$ even when approached alone. In the non-pure case, the sum of implied probabilities is continuous and strictly decreasing in $\lambda$, exceeds one at $\pi_{(1)}/N$, and equals zero at $\pi_{(1)}$. The pure-corner condition is exactly the no-deviation condition from the top box to the second-best box.
\end{proof}

The rule is water-filling on the posterior. Searchers flood boxes above a common value threshold, but they do not impose a top-$N$ count cutoff and they still collide. This is the invisible hand's partial substitute for portfolio assignment.

\begin{proposition}[Welfare identity]
\label{prop:welfare}
At the anonymous symmetric equilibrium,
\begin{equation}
G_{\mkt}^{\mathrm{sym}}
=\sum_b\pi_b\bigl[1-(1-s_b)^N\bigr]
=N\lambda.
\label{eq:welfare-identity}
\end{equation}
Moreover, $G_{\mkt}^{\mathrm{sym}}\ge\max_b\pi_b$, with strict inequality unless the equilibrium is pure.
\end{proposition}

\begin{proof}
For an active box, $\pi_b[1-(1-s_b)^N]=N s_b\pi_b\phi(s_b)=Ns_b\lambda$. Summing and using $\sum_bs_b=1$ gives $N\lambda$. If $b^*$ is a posterior mode, then
\[
N\lambda=\pi_{b^*}\frac{1-(1-s_{b^*})^N}{s_{b^*}}\ge\pi_{b^*},
\]
because $1-(1-s)^N\ge s$. Equality requires $s_{b^*}=1$.
\end{proof}

Thus equal sharing fixes the most extreme consensus pathology posterior by posterior: congestion makes repetition costly. But the game can still under-cover the posterior tail.

\subsection{The exact price of anarchy}

The planner assigns one searcher to each of the $N$ highest-posterior boxes, giving
\[
G_{\opt}=\sum_{b=1}^{\min\{N,M\}}\pi_{(b)}.
\]
Following \citet{KoutsoupiasPapadimitriou1999}, define $\PoA=G_{\opt}/G_{\mathrm{eq}}$ for the worst equilibrium. The following sharp bound is a specialization of the uniform-sharing covering-game result in \citet{Gairing2009}; a direct proof is short enough to include.

\begin{theorem}[Exact mixed price of anarchy]
\label{thm:poa}
For every posterior $\pi$ and every $N$-searcher equal-split game,
\begin{equation}
\PoA\le 2-\frac{1}{N}.
\label{eq:poa}
\end{equation}
The bound applies to mixed Nash equilibria and is tight.
\end{theorem}

\begin{proof}
Let $O=\{o_1,\ldots,o_N\}$ be an optimal set, assigning one optimal box $o_i$ to each searcher $i$ (ignore zero-mass padding if $M<N$). For any realized action profile $a$, write $W(a)$ for covered posterior mass. For each $o_i$,
\begin{equation}
\pi_{o_i}
\le \left(1-\frac{1}{N}\right)\pi_{o_i}\1\{o_i\text{ is occupied in }a\}
+u_i(o_i,a_{-i}).
\label{eq:smooth-step}
\end{equation}
If $o_i$ is empty, deviating to it gives $\pi_{o_i}$. If it is occupied, deviating gives at least $\pi_{o_i}/N$, which completes the right-hand side. Summing \cref{eq:smooth-step} over $i$ gives
\[
G_{\opt}\le\left(1-\frac{1}{N}\right)W(a)+\sum_i u_i(o_i,a_{-i}),
\]
because the occupied optimal boxes contribute at most $W(a)$. At a mixed Nash equilibrium, each player's expected current payoff is at least the payoff from the fixed deviation to $o_i$. Total expected current payoff equals $\E[W(a)]$. Hence
\[
G_{\opt}\le\left(2-\frac{1}{N}\right)\E[W(a)].
\]

For tightness, let one ``decoy'' box have posterior mass $1/2$ and let $N$ tail boxes each have mass $1/(2N)$. The profile in which all searchers choose the decoy is a Nash equilibrium: each receives $1/(2N)$ and is indifferent to deviating to an empty tail box. Equilibrium discovery is $1/2$, while the planner covers the decoy and $N-1$ tail boxes, giving $1-1/(2N)$. The ratio is $2-1/N$.
\end{proof}

The theorem distinguishes two questions. A posterior may possess an efficient pure equilibrium and an inefficient anonymous mixed equilibrium at the same time. The price of anarchy controls the worst strategic outcome; the symmetric benchmark measures what happens under anonymous, uncoordinated mixing.

The proof is a smoothness argument in the sense of \citet{Roughgarden2015}: inequality \cref{eq:smooth-step} establishes that the game is $(1,\,1-1/N)$-smooth, so the extension theorems for robust price-of-anarchy bounds apply verbatim. The bound $2-1/N$ therefore holds not only for mixed Nash equilibria but for all coarse correlated equilibria and for the average play of any no-regret learning dynamics---agents need not compute an equilibrium for the guarantee to bind.

\paragraph{Uniform posterior.}
If posterior mass is uniform on $T$ boxes, the anonymous symmetric equilibrium is $s_b=1/T$ and
\begin{equation}
G_{\mkt}^{\mathrm{sym}}=1-\left(1-\frac{1}{T}\right)^N,
\qquad
G_{\opt}=\frac{\min\{N,T\}}{T}.
\label{eq:uniform}
\end{equation}
When $T=N$,
\begin{equation}
\frac{G_{\opt}}{G_{\mkt}^{\mathrm{sym}}}
=\frac{1}{1-(1-1/N)^N}
\longrightarrow \frac{e}{e-1}\approx1.582.
\label{eq:econstant}
\end{equation}
For $N=8$, the ratio is $1.523482$. An efficient pure equilibrium placing one searcher on each box also exists. The $e/(e-1)$ constant is therefore the collision cost of symmetric mixing, not a universal impossibility result.

\subsection{The canonical four-protocol ordering}

For every realized pooled posterior in the sixteen-box benchmark, we solve \cref{eq:waterfill} and then average over the exact distribution of report-count classes. The resulting anonymous market discovery is
\begin{equation}
G_{\mkt}^{\mathrm{sym}}=0.59909925.
\label{eq:canonical-market}
\end{equation}
It generates $2.673494$ distinct actions in expectation. The complete canonical comparison is shown in \cref{fig:four-protocols} and \cref{tab:canonical}.

\begin{figure}[t]
\centering
\includegraphics[width=0.92\linewidth]{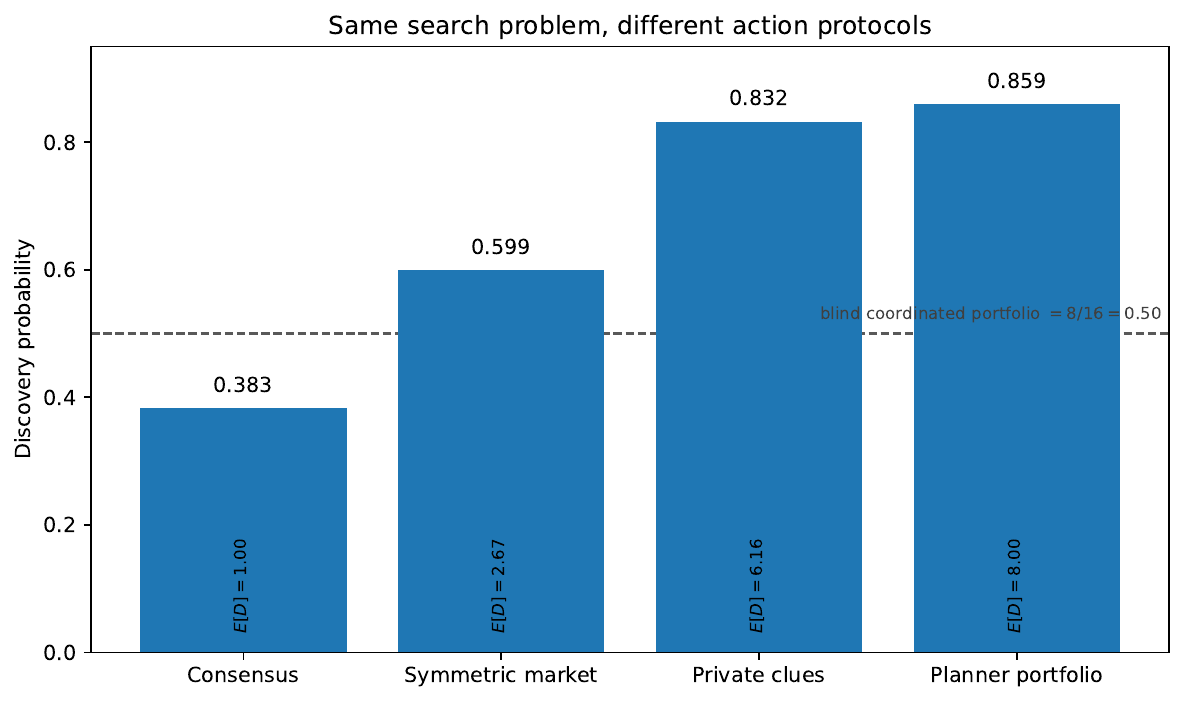}
\caption{Exact ex-ante discovery under four protocols in the independent sixteen-box benchmark. The labels inside the bars report the expected number of distinct boxes searched. The dashed line marks the blind coordinated portfolio at $8/16=0.50$: informed consensus falls below a protocol that ignores the clues entirely. Equal sharing supplies partial dispersion, but not enough to match natural coverage from independent private clues.}
\label{fig:four-protocols}
\end{figure}

\begin{table}[t]
\centering
\caption{Canonical protocol comparison: $M=16$, $N=8$, $p=0.20$.}
\label{tab:canonical}
\begin{threeparttable}
\begin{tabular}{lllccr}
\toprule
Protocol & Information & Allocation rule & Quality $Q$ & Discovery $G$ & $\E[$distinct$]$\\
\midrule
Consensus & pooled & repeat posterior mode & 0.383469 & 0.383469 & 1.000\\
Symmetric market & pooled & equal-split water-filling & 0.344136 & 0.599099 & 2.673\\
Private clues & local & follow own report & 0.200000 & 0.832228 & 6.157\\
Planner portfolio & pooled & top eight, one each & 0.107428 & 0.859421 & 8.000\\
\bottomrule
\end{tabular}
\begin{tablenotes}[flushleft]\footnotesize
\item $Q$ is the average standalone probability that a selected box contains the jackpot, as in \cref{eq:redundancy}. The symmetric-market row is an equilibrium-selection benchmark, not a claim that every Nash equilibrium has the same value.
\end{tablenotes}
\end{threeparttable}
\end{table}
\FloatBarrier

The table records an exact inversion:
\[
Q_{\cons}>Q_{\mkt}^{\mathrm{sym}}>Q_{\priv}>Q_{\plan}
\qquad\text{while}\qquad
G_{\cons}<G_{\mkt}^{\mathrm{sym}}<G_{\priv}<G_{\plan}.
\]
Moving toward the group-optimal allocation requires assigning actions that are individually less likely to be correct in isolation: the planner raises group discovery to $0.8594$ while reducing the average standalone probability that an assigned box contains the jackpot to $0.1074$. Average quality should not be confused with equilibrium payoff. By \cref{eq:welfare-identity}, the symmetric market payoff per searcher is $\lambda=G_{\mkt}^{\mathrm{sym}}/N=0.074887$. The quality number says the market concentrates on individually attractive boxes; the payoff number says crowding sharply dilutes the private return from doing so.

Congestion pricing cures unanimity but not concentration. Private clue-followers condition on different signals and therefore obtain substantial natural coverage. Symmetric market participants condition on the same posterior and rationally overweight its peak. The invisible hand disperses, but toward the wrong distribution.

\section{Reward the rescue, not the win}
\label{sec:mechanism}

The equal-split rule rewards an agent for being on the jackpot box. The group values an agent for being on a jackpot box that no one else covered. The difference is marginal contribution.

Consider the \emph{sole-rescue rule}: if exactly one searcher chose the realized target, that searcher receives the unit prize; if two or more chose the target, no one is paid. Conditional on posterior $\pi$, a searcher choosing box $b$ receives expected payoff $\pi_b$ if alone on $b$ and zero otherwise.

\begin{proposition}[Pure-strategy implementation]
Suppose $M\ge N$ and every posterior mass is positive. Under the sole-rescue rule:
\begin{enumerate}[label=(\roman*)]
\item the game is an exact potential game with potential equal to covered posterior mass;
\item every pure Nash equilibrium is collision-free; and
\item the occupied boxes in every pure Nash equilibrium are exactly an $N$-box posterior-optimal portfolio.
\end{enumerate}
Thus every pure Nash equilibrium attains $G_{\opt}$.
\end{proposition}

\begin{proof}
If a player leaves a box, covered mass falls by $\pi_b$ exactly when the player was its sole occupant. If the player enters another box, covered mass rises by $\pi_c$ exactly when that box was empty. The player's payoff change therefore equals the change in covered mass.

If a pure profile contains a collision, at least one box is empty because $M\ge N$. A colliding player earns zero and can move to an empty positive-mass box for a strictly positive payoff, so no collision can persist. In a collision-free profile, a player on box $b$ earns $\pi_b$. Moving to an occupied box yields zero; moving to an empty box $c$ yields $\pi_c$. The profile is a Nash equilibrium if and only if every occupied box has posterior mass at least as large as every empty box. Its occupied set is therefore a top-$N$ set.
\end{proof}

The result is full implementation in pure Nash equilibrium, not dominant-strategy implementation. A searcher's best action still depends on which boxes others occupy. Nor does the result eliminate all mixed-equilibrium or equilibrium-selection concerns. A public role-assignment or sequential convention may still be useful for selecting among efficient top-$N$ matchings.

The rule is weakly budget-balanced ex post: it never pays more than the unit prize. On the efficient equilibrium path it pays the full prize exactly when discovery occurs. Off equilibrium it burns the prize on a duplicated discovery. In richer applications, attribution, side payments, risk aversion, and collusion may make literal sole-rescue contracting difficult. The benchmark point is narrower: compensating unique coverage aligns the potential with the organization's discovery objective, whereas equal sharing aligns total payout with discovery but not the allocation of effort.

\section{Common-source dependence}
\label{sec:copying}

Real searchers rarely receive conditionally independent signals. They read the same reports, use the same models, attend the same meetings, and inherit the same narratives. We introduce one latent common source to make this dependence explicit.

\subsection{Copying and channel collapse}

Let $X_0$ be a common cue with the same accuracy law as a private clue. Each searcher independently copies $X_0$ with probability $c$ and otherwise draws an independent clue:
\[
X_i=
\begin{cases}
X_0, & \text{with probability }c,\\
\widetilde X_i, & \text{with probability }1-c.
\end{cases}
\]
The copying probability $c$ is a scalar parameter---distinct from the report-count vector $C=(C_1,\ldots,C_M)$ of \cref{sec:signals}---and it is not the pairwise correlation coefficient. For correctness indicators $Y_i=\1\{X_i=\theta\}$,
\[
\mathrm{Corr}(Y_i,Y_j\mid\theta)=c^2.
\]

Let $K\sim\mathrm{Binomial}(N,c)$ be the number of copiers. Conditional on $K$, the copiers contribute one common discovery channel if $K\ge1$, while the $N-K$ noncopiers contribute independent channels.

\begin{proposition}[Channel collapse]
The effective channel count is
\begin{equation}
N_{\mathrm{eff}}=N-K+\1\{K\ge1\}.
\label{eq:effective-channels}
\end{equation}
Private report-following discovery is
\begin{equation}
G_{\priv}(c)
=1-p\bigl[(1-c)(1-p)\bigr]^N
-(1-p)\bigl[1-p(1-c)\bigr]^N.
\label{eq:private-copying}
\end{equation}
\end{proposition}

\begin{proof}
If at least one searcher copies, private search misses only if the common cue misses and every noncopier's independent clue misses. If no searcher copies, all $N$ independent clues must miss. Averaging over $K$ and collecting terms yields \cref{eq:private-copying}.
\end{proof}

At $c=0$, \cref{eq:private-copying} is $1-(1-p)^N$. At $c=1$, it is $p$: eight nominal searchers have become one effective clue. In the canonical model, expected distinct proposals fall from $6.157$ at $c=0$ to $1.362$ at $c=0.95$.

\subsection{Agreement need not mean posterior rank}

Under independence, report count is a sufficient ranking statistic. Under copying, a large count can be generated by one incorrect source repeated many times. Exact marginalization over the latent cue and copier count can reverse the ranking.

At $c=0.4$, consider counts $(6,2,0,\ldots,0)$. Let $a$ be the six-count box and $b$ the two-count box. The exact count likelihoods are
\[
\Prob(C\mid\theta=a)=1.1427\times10^{-4},
\qquad
\Prob(C\mid\theta=b)=1.3977\times10^{-4}.
\]
With a uniform prior, the two-count box has the larger posterior. This does not claim that count ranking generally fails dramatically; it establishes that dependence-aware ranking is conceptually distinct from raw agreement. In the canonical calculations, the Bayes-versus-count discovery discount is concentrated at one- and two-action budgets and vanishes for $L\ge3$ on the tested grid.

\subsection{The centralized planner gain}

Let $D$ be the number of distinct privately proposed boxes and let $V_N(c)$ be the Bayes-optimal $N$-action portfolio value after observing reports generated at copying probability $c$. Define the equal-budget planner gain
\begin{equation}
\Delta_{\plan}(c)=V_N(c)-G_{\priv}(c).
\label{eq:planner-gain}
\end{equation}

\begin{theorem}[Monotone planner gain]
In the atomic $M$-state model with $2\le N\le M$ and $p>1/M$, the planner gain $\Delta_{\plan}(c)$ is strictly increasing for $c\in(0,1)$. At full copying,
\begin{equation}
\Delta_{\plan}(1)=(N-1)q.
\label{eq:planner-limit}
\end{equation}
\end{theorem}

The proof is in \cref{app:planner-proof}. Its key identity is
\begin{equation}
\Delta_{\plan}(c)
=\E\left[\1\{\text{private search misses}\}\frac{N-D}{M-D}\right].
\label{eq:fill-identity}
\end{equation}
All named boxes precede unnamed boxes in the posterior ranking, so the planner preserves every distinct proposal and fills unused slots with otherwise unsearched states. Conditional on a miss, unnamed states are symmetric. Increasing the number of copiers weakly shrinks $D$ and can convert a success into a miss; both changes raise the value of the planner's fill slots. The binomial copier count shifts upward with $c$.

\begin{figure}[t]
\centering
\includegraphics[width=0.84\linewidth]{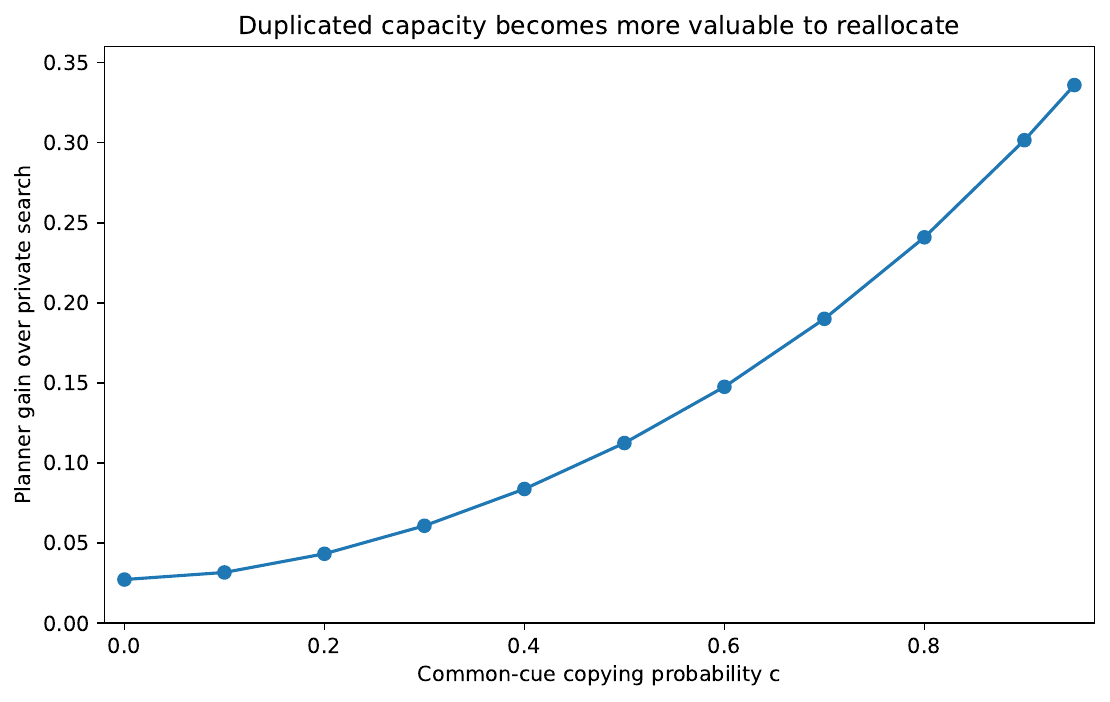}
\caption{Exact canonical planner gain $V_8(c)-G_{\priv}(c)$. As common-source copying creates more duplicated capacity, identifying and reallocating it becomes more valuable.}
\label{fig:planner-gain}
\end{figure}

For the canonical parameters, the gain rises from $0.027193$ at $c=0$ to $0.335927$ at $c=0.95$, with limit $0.373333$. The recovery budget falls from seven coordinated actions at $c=0$ to two at $c=0.95$. The theorem concerns a frictionless centralized benchmark: it combines pooled observation with assignment authority and therefore upper-bounds what a decentralized institution may realize.

\subsection{Selfish search under copying}

Now let searchers pool their reports, form the herd-aware posterior by marginalizing over the latent cue and copier count, and play the equal-split game on that posterior. For each report-count class, we solve the anonymous equilibrium and average conditional discovery over its exact probability.

\begin{table}[t]
\centering
\caption{Canonical outcomes under common-cue copying.}
\label{tab:copying}
\begin{tabular}{rcccc}
\toprule
$c$ & Consensus & Symmetric market & Private reports & Planner\\
\midrule
0.00 & 0.383 & 0.599 & 0.832 & 0.859\\
0.50 & 0.271 & 0.552 & 0.655 & 0.768\\
0.70 & 0.240 & 0.477 & 0.512 & 0.702\\
0.80 & 0.221 & 0.428 & 0.423 & 0.664\\
0.90 & 0.206 & 0.388 & 0.319 & 0.621\\
0.95 & 0.202 & 0.380 & 0.262 & 0.598\\
\bottomrule
\end{tabular}
\end{table}

At low copying, private report-following wins because different reports supply natural coverage. As copying rises, that advantage collapses. The market continues to disperse over the support of the common posterior and therefore degrades more slowly. High-precision enumeration gives the canonical crossover
\begin{equation}
c^*=0.7884616565,
\qquad
G_{\mkt}^{\mathrm{sym}}(c^*)=G_{\priv}(c^*)=0.4338934643.
\label{eq:crossover}
\end{equation}
At full copying, private report-following and consensus both equal $p=0.20$, the planner reaches $p+(N-1)q=0.573333$, and the anonymous equal-split market reaches approximately $0.382006$.

The crossover is not knife-edge in the parameters. A comparative-statics scan over $72$ combinations---$M\in\{8,12,16,24,32\}$, $N\in\{4,6,8\}$, $p\in\{0.10,\ldots,0.30\}$ with $p>1/M$, and $c$ scanned in steps of $0.02$ up to $0.98$---finds that the market--private gap is negative at $c=0$ in every combination and changes sign at most once: $52$ combinations cross exactly once, at crossovers ranging from $c\approx0.29$ to $c\approx0.98$, while in the remaining $20$, concentrated at large $M$, small $N$, and strong signals, private report-following stays ahead throughout the scanned range. On this grid the qualitative pattern is uniform---natural coverage first, forced dispersion after enough channel collapse---though single crossing remains a computational pattern rather than a theorem.

\begin{figure}[t]
\centering
\includegraphics[width=0.94\linewidth]{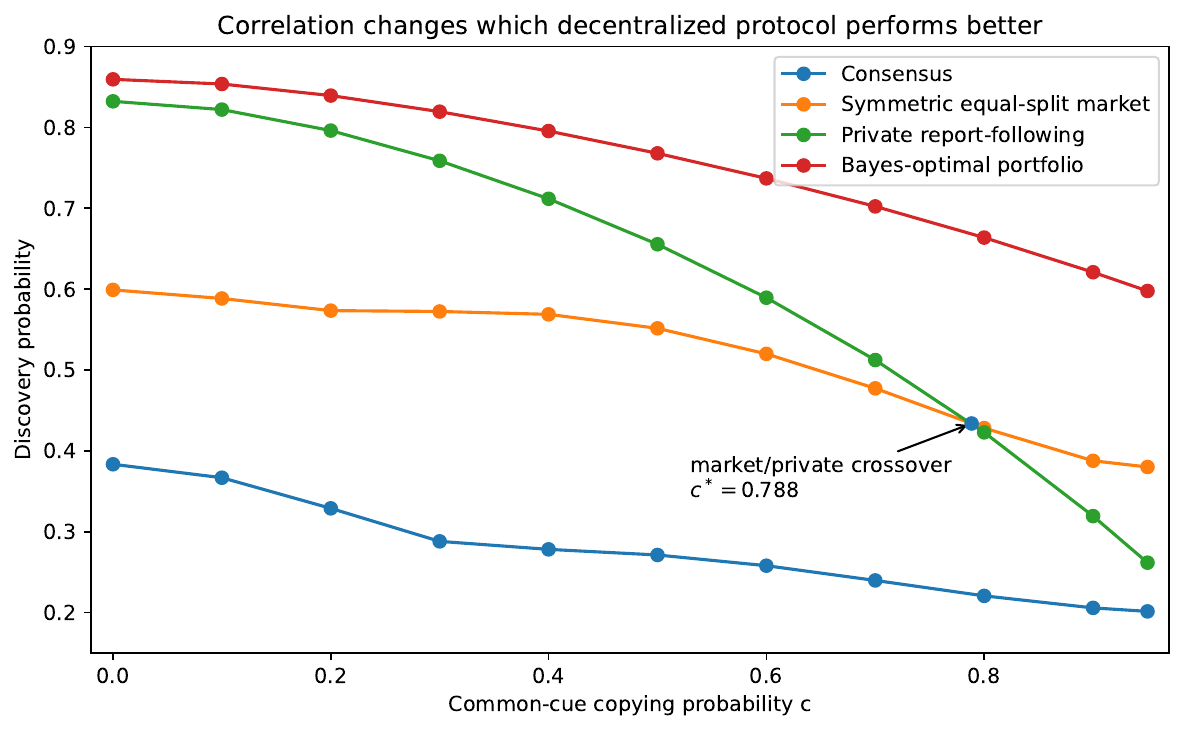}
\caption{Exact canonical comparison under common-source copying. The symmetric equal-split market initially trails decentralized report-following, then overtakes it when channel collapse becomes sufficiently severe. The crossover is a canonical computational result, not a general theorem.}
\label{fig:copying-protocols}
\end{figure}

The institutional implication is conditional. When information sources are genuinely independent, decentralized local action can dominate a common-posterior market. When sources are highly correlated, some mechanism that forces dispersion can outperform nominal decentralization. Correlation raises the value not only of a perfect planner but also of imperfect coordination.
\FloatBarrier

\section{Why this is a benchmark}
\label{sec:canonical}

A useful benchmark is not a universal law. It is a small model that isolates a mechanism, supplies stable vocabulary, and produces quantities that richer models can preserve or overturn. The sixteen-box environment is designed for that role.

\begin{table}[t]
\centering
\caption{The benchmark dictionary.}
\label{tab:dictionary}
\begin{tabularx}{\linewidth}{>{\raggedright\arraybackslash}p{0.24\linewidth}X}
\toprule
Object & Interpretation\\
\midrule
One-action value $V_1(\F)$ & How accurately the information architecture identifies its best single target.\\
Portfolio frontier $V_L(\F)$ & The discovery attainable from $L$ actions under the same posterior.\\
Protocol loss $\Loss_L$ & Attainable posterior mass destroyed by the action rule.\\
Recovery budget $L^*$ & Smallest coordinated budget that replaces decentralized search.\\
Effective channels & Number of independent chances encoded by nominal reports.\\
Symmetric market value & Dispersion supplied by equal sharing without role assignment.\\
Price of anarchy & Worst equilibrium loss relative to the top-$N$ portfolio.\\
Sole-rescue implementation & Alignment obtained by rewarding unique coverage.\\
Correlation crossover $c^*$ & Point at which forced dispersion overtakes natural private coverage.\\
\bottomrule
\end{tabularx}
\end{table}

The benchmark separates four mechanisms that richer environments often confound:
\begin{enumerate}
\item \textbf{Information improvement.} Pooling raises the quality of the posterior ranking.
\item \textbf{Action compression.} A one-answer protocol makes nominally separate actions perfectly correlated.
\item \textbf{Action synthesis.} A coordinator can redirect duplicated capacity to candidates no private clue proposed.
\item \textbf{Channel merging.} Shared sources reduce effective diversity before any explicit consensus rule is applied.
\end{enumerate}

\subsection{Formal lineage}

\paragraph{Team theory and optimal search.}
The common-objective model is a team problem in the sense of \citet{MarschakRadner1972}: agents possess dispersed information and choose a joint action profile. Blackwell's comparison of experiments underlies the value-of-information result \citep{Blackwell1953}. The operational parent is optimal search: \citet{Koopman1957} studies how finite search effort should be allocated over uncertain target locations. The atomic model is a discrete, all-or-nothing corner in which the posterior top-$L$ rule is exact.

\paragraph{Parallel R\&D and organizational learning.}
Parallel projects hedge uncertainty in innovation \citep{Nelson1961}. Organizational-learning and division-of-cognitive-labor models show how imitation and communication can reduce exploration \citep{March1991,Kitcher1990}. Network epistemology makes the same tension endogenous to communication structure \citep{Zollman2007,MayoWilsonEtAl2011}. In the present accounting, limited communication is a second-best instrument: sparse networks preserve diversity only by withholding information value, whereas the planner and rescue benchmarks keep the shared posterior intact and repair the action allocation directly, separating the two effects that communication limits bundle. \citet{DuedeEvans2024} independently describe the belief-side half of this trade-off---collaboration inflates collective certainty while the likelihood of genuine replication falls---which is the accuracy--coverage reversal stated in belief space rather than action space. \citet{RzhetskyEtAl2015} study a much richer dynamic problem of experiment choice on an evolving scientific network. The present benchmark strips away dynamics and careers to expose a finite-action accounting identity: a community can improve its shared ranking and still waste its experimental portfolio.

\paragraph{Herding and correlated judgment.}
Information-cascade models explain how rational observation can generate convergent beliefs or actions \citep{Banerjee1992,BikhchandaniEtAl1992,SmithSorensen2000}. Correlated-jury work shows that dependence weakens voting aggregation \citep{Ladha1992}. Our copying model does not claim the first failure of raw agreement. It places dependence inside a multi-alternative discovery problem and measures three distinct consequences: channel collapse, posterior-versus-count ranking loss, and the value of reallocating duplicated actions.

\paragraph{Congestion, covering, and potential games.}
The equal-split extension is a singleton congestion game with a harmonic potential \citep{Rosenthal1973,MondererShapley1996}. Its welfare is a coverage function, linking it to valid-utility and covering games \citep{Vetta2002,Gairing2009}. This lineage is important for claim calibration: the exact $2-1/N$ bound is not presented as a new general result in algorithmic game theory. What is new in the benchmark is the way the covering game is generated endogenously by a Bayesian posterior and compared ex ante with private signals, consensus, a planner, and common-source copying. A complementary line studies how information itself interacts with congestion: \citet{AcemogluEtAl2018} show that giving drivers additional route information can raise equilibrium travel times---an informational Braess' paradox---and characterize the network class immune to it, with a subsequent information-design literature choosing what to disclose \citep{DasKamenicaMirka2017}. The mechanism is instructively different. There, information expands strategy sets in a delay-minimization routing game and harms through re-routing externalities in equilibrium; here, the information architecture is held fixed, richer information weakly raises the attainable frontier (the argument of \cref{prop:voi} applies at every budget), and the loss arises entirely from the protocol that converts the posterior into actions. Whether a sharper shared posterior can lower equilibrium coverage through the water-filling pure corner---an informational Braess' paradox in discovery rather than delay---is a disclosure-stage question we leave open.

\paragraph{Submodular extensions.}
With one atomic target, the value of a set of distinct boxes is modular. If actions cover overlapping state sets, the objective becomes monotone submodular; if actions reveal information sequentially, adaptive submodularity becomes relevant \citep{NemhauserEtAl1978,GolovinKrause2011}. The exact top-$L$ portfolio would then be replaced by approximation guarantees, but protocol loss and recovery budgets would remain well-defined.

\section{Organizational implications}
\label{sec:implications}

\paragraph{Research and development.}
A firm may combine analyst reports, technical reviews, and shared models to improve its ranking of projects, then route too much laboratory capacity toward the same leading candidate. The benchmark suggests treating information aggregation and project assignment as separate design problems. A portfolio manager, explicit capacity reservation for minority hypotheses, or marginal-contribution credit can preserve the value of shared information without preserving disagreement for its own sake.

\paragraph{Venture capital and innovation portfolios.}
Shared diligence can improve judgment about the modal opportunity while concentrating capital and sourcing effort in the same companies or categories. The model is not a literal theory of venture returns: it omits prices, ownership, multiple winners, heavy tails, and endogenous entry. It does yield a testable comparative static. As sourcing and diligence become more correlated, the marginal value of portfolio-level differentiation should rise.

\paragraph{Science.}
Open communication can improve beliefs about promising hypotheses while narrowing the set of experiments performed. Restricting communication can preserve diversity, but only by withholding information. The planner and sole-rescue benchmarks point to a different design: preserve the shared posterior and diversify the actions. Provenance data could help distinguish many nominal reports from many independent channels.

\paragraph{Multi-agent artificial intelligence.}
Agents connected to the same model, memory, retrieval context, or evaluation rubric may produce individually stronger but highly correlated recommendations. Multi-agent systems should therefore distinguish ensemble size from effective channel count. Posterior sampling, explicit role differentiation, diverse retrieval sources, and marginal-coverage rewards are natural analogues of the benchmark's portfolio mechanisms. The smoothness form of \cref{thm:poa} is directly relevant here: ensembles of adaptive agents running any regret-minimizing update inherit the $2-1/N$ coverage guarantee under equal-split credit, without ever computing an equilibrium. The blind-sentinel negative result is equally actionable: adding uninformed explorer agents to an informed ensemble is dominated in the canonical environment, so useful diversity must be signal-contingent---sampled from the posterior or assigned to distinct evidence---rather than signal-free.

The empirical program is straightforward in principle. Estimate posterior or score distributions, measure repeated-action patterns and source overlap, compare observed coverage with the top-$L$ frontier, and report protocol loss and the recovery budget. The copying parameter is a model, not an observable primitive, but source provenance and correlated failure data can discipline it. So can duplication itself: the expected number of distinct proposals is strictly decreasing in $c$ on the computed grid (from $6.157$ at $c=0$ to $1.362$ at $c=0.95$ in the canonical environment), so matching an observed duplication level to the benchmark's exact curve pins down a latent copying rate, which in turn implies a recovery budget and a planner gain. This is the sense in which the benchmark's quantities serve as reference points for richer models: comparison and inversion, not structural estimation.

\section{Limitations and extensions}
\label{sec:limitations}

The benchmark is intentionally narrow. Exactly one atomic state is valuable; each action tests one state; actions are simultaneous and costless; private clues are ex ante symmetric; the planner can observe and assign reports; equal-split players share a common posterior; and dependence is generated by one latent cue whose provenance is hidden. These assumptions make both inference and assignment unusually clean.

Several qualifications matter for interpretation. First, the canonical market value uses the anonymous symmetric equilibrium. Equal-split games can also have efficient asymmetric pure equilibria, so role labels, conventions, or sequential play may materially change outcomes. Second, sole-rescue gives full pure-equilibrium implementation but not dominant strategies or a guarantee for every mixed equilibrium. Third, the monotone planner-gain theorem is exact for the stated copying process; the market/private crossover is a canonical computational result. Fourth, the planner benchmark changes both information access and decision rights relative to decentralized clue-following.

One intermediate cell deserves explicit mention. A decentralized-team benchmark would allow the agents to commit ex ante to role-dependent policies $f_1,\ldots,f_N$ while retaining private, noncommunicated clues, with team value $\max_{f_1,\ldots,f_N}\Prob\bigl(\theta\in\{f_1(X_1),\ldots,f_N(X_N)\}\bigr)$. Characterizing that private-information team optimum would separate assignment authority from information pooling, which the planner benchmark changes jointly. We leave it open rather than insert a hand-selected policy: the simple hybrid that replaces clue-followers with blind sentinels is maximized at zero sentinels in the canonical environment, so the interesting policies are signal-contingent.

Natural extensions include heterogeneous searchers, multiple common sources or cliques, endogenous information acquisition, multiple targets, costly or capacity-constrained actions, heavy-tailed prizes, sequential experimentation, overlapping coverage, collusion under rescue rewards, and mechanism design when the posterior itself is privately held. In these environments the numerical yardsticks may change, but the benchmark's separation between information, allocation, incentives, and dependence remains available.

\section{Conclusion}

Collective discovery is not collective estimation with more participants. Estimation asks for one answer. Discovery allocates a portfolio of attempts.

The sixteen-box benchmark makes the distinction exact. Pooling eight imperfect clues raises the best single recommendation from $20.0\%$ to $38.35\%$. Repeating that recommendation lowers discovery from $83.22\%$ to $38.35\%$. A planner using the same pooled reports reaches $85.94\%$, and seven coordinated actions recover the private benchmark. Equal sharing supplies partial dispersion and raises discovery to $59.91\%$, but its worst equilibrium can lose a factor $2-1/N$. Rewarding unique rescue aligns every pure equilibrium with the first-best portfolio. Common-source copying collapses natural coverage, increases the value of reallocation, and eventually makes even an imperfect dispersing market outperform private report-following.

The central design lesson is therefore not ``share less information.'' It is: treat information sharing and action allocation as distinct institutional choices. Share the evidence. Diversify the actions.

\appendix

\section{Exact independent enumeration}
\label{app:enumeration}

By symmetry, condition on box $1$ being the target. The report-count vector
\[
C=(C_1,\ldots,C_M)\sim\mathrm{Multinomial}(N;p,q,\ldots,q).
\]
For a count vector $x$, let
\[
r(x)=|\{b:x_b>x_1\}|,
\qquad
t(x)=|\{b:x_b=x_1\}|.
\]
With uniform randomization at cutoff ties, the conditional probability that the true box belongs to a top-$L$ count set is
\begin{equation}
\tau_L(x)=
\begin{cases}
0, & r(x)\ge L,\\
1, & r(x)+t(x)\le L,\\
\dfrac{L-r(x)}{t(x)}, & r(x)<L<r(x)+t(x).
\end{cases}
\label{eq:tie-inclusion}
\end{equation}
Hence the exact pooled frontier is
\begin{equation}
G_{M,N,L}(p)
=\sum_{x_1+\cdots+x_M=N}
\frac{N!}{\prod_{b=1}^M x_b!}
 p^{x_1}q^{N-x_1}\tau_L(x).
\label{eq:multinomial-frontier}
\end{equation}
For $M=16$ and $N=8$, the sum contains
\[
\binom{N+M-1}{M-1}=\binom{23}{15}=490{,}314
\]
count vectors. Grouping by permutations of the fifteen non-target boxes reduces both the independent and common-cue calculations to $67$ count classes.

At $L=7$, the difference between the pooled portfolio and private clue-following can be decomposed as
\[
\Prob(\theta\in S^*_7\setminus S_{\priv})
-\Prob(\theta\in S_{\priv}\setminus S^*_7).
\]
Exact enumeration gives a synthesis gain of $0.011744$ and a compression loss of $0.007962$. At $L=8$, no box named by a private clue is discarded; the improvement $0.027193$ comes entirely from reallocating duplicated capacity.

\section{Common-cue likelihood and Bayes ranking}
\label{app:likelihood}

For a candidate target $t$, define the report probability vector
\[
\rho_r^{(t)}=
\begin{cases}
p,&r=t,\\q,&r\neq t.
\end{cases}
\]
Let $z$ be the common cue and $k$ the copier count. For observed counts $x$,
\begin{align}
&\Prob(C=x\mid\theta=t,X_0=z,K=k)\\
&\quad=
\1\{x_z\ge k\}
\frac{(N-k)!}{(x_z-k)!\prod_{r\neq z}x_r!}
\prod_{r=1}^M\left(\rho_r^{(t)}\right)^{x_r-k\1\{r=z\}}.
\end{align}
Marginalizing gives
\begin{align}
\Prob(C=x\mid\theta=t,c)
=\sum_{z=1}^M\rho_z^{(t)}
\sum_{k=0}^N\binom Nk c^k(1-c)^{N-k}
\Prob(C=x\mid\theta=t,X_0=z,K=k).
\label{eq:common-likelihood}
\end{align}
With a uniform prior, normalize these $M$ likelihoods to obtain the herd-aware posterior. The verification script precomputes the coefficient of every basis term $c^k(1-c)^{N-k}$ for each count class and candidate target.

\section{Proof of the monotone planner-gain theorem}
\label{app:planner-proof}

We give the omitted details for \cref{eq:fill-identity} and monotonicity.

\begin{lemma}[Named states precede unnamed states]
For every realizable report vector under the common-cue model, every state with a positive report count has posterior probability at least as large as every state with zero reports.
\end{lemma}

\begin{proof}
Fix a named state $b$ and an unnamed state $d$. Compare the mixture likelihood under hypotheses $\theta=b$ and $\theta=d$ term by term in \cref{eq:common-likelihood}. If the common cue is neither $b$ nor $d$, the cue prefactor is identical while the residual multinomial likelihood differs by a nonnegative power of $p/q$ in favor of $b$. If the cue is $b$, its prefactor is $p$ under $\theta=b$ and $q$ under $\theta=d$, and the residual likelihood contributes another nonnegative power of $p/q$. If the cue is $d$, feasibility requires $k=0$ because $d$ is unnamed; the combined likelihood ratio is $(p/q)^{x_b-1}\ge1$. Every matched nonzero term weakly favors the named state. Summing over cues and copier counts preserves the inequality.
\end{proof}

The optimal $N$-action portfolio therefore includes all $D$ distinct named states and fills $N-D$ positions with unnamed states. On the event that private report-following misses, the target lies among the $M-D$ unnamed states. Symmetry assigns equal conditional posterior mass to those states, so the planner's conditional rescue probability is $(N-D)/(M-D)$. This proves \cref{eq:fill-identity}.

Condition on copier count $K=k$. Given $k$, the joint law of the common cue and the $N-k$ independent clues does not otherwise depend on $c$. Define
\[
h(k)=\E\left[\left.\1\{\text{miss}\}\frac{N-D}{M-D}\right|K=k\right].
\]
Then $\Delta_{\plan}(c)=\E[h(K)]$ for $K\sim\mathrm{Binomial}(N,c)$. We have $h(0)=h(1)$ because one copy of $X_0$ has the same joint distribution as one independent clue. For $k\ge1$, couple the $k$-copier system and the $(k+1)$-copier system by converting one independent searcher into an additional copier while holding the common cue and all other reports fixed. Because the common cue is already represented, the conversion can only weakly reduce $D$. It can only weakly increase the miss indicator. The fraction $(N-D)/(M-D)$ is nondecreasing as $D$ falls when $N\le M$. Thus $h(k+1)\ge h(k)$. The inequality is strict for at least one $k$: with positive probability, the converted independent clue is the only report naming the target while the common cue and all remaining clues miss.

If $K'\sim\mathrm{Binomial}(N-1,c)$, the standard binomial derivative identity yields
\[
\frac{d}{dc}\E[h(K)]
=N\E[h(K'+1)-h(K')]>0
\qquad(c\in(0,1)).
\]
The hypothesis $N\ge2$ is necessary: with a single searcher the named box is the posterior mode, the planner replicates private search, the gain is identically zero, and the derivative vanishes throughout $(0,1)$ because $h(0)=h(1)$.
At $c=1$, one box is named. The planner adds $N-1$ of the $M-1$ unnamed boxes, each with posterior mass $q$, while private search covers only the named box. Hence $\Delta_{\plan}(1)=(N-1)q$.

\section{Proof of the proportional sparse-evidence limit}
\label{app:scaling}

Throughout, $N=N_M$ with $N_M/M\to\alpha\in(0,1)$, and $p_M=r/(M-1+r)$, $q_M=1/(M-1+r)$ for fixed $r>1$, so that $N_Mp_M\to\alpha r$ and $N_Mq_M\to\alpha$: the target receives an asymptotically $\mathrm{Poisson}(\alpha r)$ number of reports, while each false box receives approximately $\mathrm{Poisson}(\alpha)$. All limits are as $M\to\infty$ and all clues are independent.

\emph{Blind portfolio.} The coordinator selects $N_M$ distinct boxes without observing clues. Under the uniform prior, $G_{\mathrm{blind}}=N_M/M\to\alpha$.

\emph{Private clue-following.} By \cref{eq:private}, $G_{\priv}=1-(1-p_M)^{N_M}$, and $N_M\log(1-p_M)=-N_Mp_M\bigl(1+O(p_M)\bigr)\to-\alpha r$, so $G_{\priv}\to1-e^{-\alpha r}$.

\emph{Consensus.} Condition on $\theta$. The target count $C_\theta\sim\mathrm{Binomial}(N_M,p_M)$ converges in distribution to $\mathrm{Poisson}(\alpha r)$ and is therefore bounded in probability. Conditional on the number of off-target reports $F=N_M-C_\theta$, the false counts are multinomial with $F$ trials over $M-1$ equally likely boxes, and $F/M\to\alpha$ in probability. Fix an integer $j$ and let $W_j=|\{b\neq\theta:C_b\ge j\}|$. Standard occupancy asymptotics give $\E[W_j]/M\to\Prob(\mathrm{Poisson}(\alpha)\ge j)>0$, and the variance of $W_j$ is $O(M)$ because the occupancy indicators are negatively correlated, so $W_j\to\infty$ in probability. Hence $\max_{b\neq\theta}C_b\ge j$ with probability tending to one for every fixed $j$: the largest false count diverges in probability while $C_\theta$ remains tight. Consensus success requires $C_\theta\ge\max_{b\neq\theta}C_b$ (an event whose probability tends to zero), so $G_{\cons}\to0$; uniform tie-breaking only lowers the success probability further.

\emph{Portfolio.} By \cref{eq:posterior-count}, every named box has strictly larger posterior mass than every unnamed box, and at most $N_M$ boxes are named, so the top-$N_M$ rule retains every named box and fills its remaining slots uniformly at random among the posterior-tied unnamed boxes. Let $D$ be the number of distinct named boxes. On $\{C_\theta\ge1\}$ the portfolio contains $\theta$. On the miss event $\{C_\theta=0\}$, the target is one of the $M-D$ unnamed boxes, which carry equal conditional posterior mass, so the conditional rescue probability is $(N_M-D)/(M-D)$. This is the $c=0$ case of \cref{eq:fill-identity}:
\[
G_{\mathrm{portfolio}}
=G_{\priv}+\E\!\left[\1\{C_\theta=0\}\,\frac{N_M-D}{M-D}\right].
\]
Conditional on the miss event, $D$ is the number of occupied boxes when $N_M$ balls are placed uniformly over $M-1$ boxes. Occupancy concentration ($\E[D]=(M-1)[1-(1-\tfrac{1}{M-1})^{N_M}]$ with variance $O(M)$) gives $D/M\to1-e^{-\alpha}$ in probability, hence
\[
\frac{N_M-D}{M-D}
\;\longrightarrow\;
\tau_\alpha
=\frac{\alpha-(1-e^{-\alpha})}{e^{-\alpha}}
=1-(1-\alpha)e^{\alpha}
\qquad\text{in probability},
\]
where $\tau_\alpha\in(0,1)$ because $e^{-\alpha}>1-\alpha$ for $\alpha>0$: duplicated reports leave free slots asymptotically. The ratio is bounded by one, so bounded convergence together with $\Prob(C_\theta=0)\to e^{-\alpha r}$ yields
\[
G_{\mathrm{portfolio}}
\longrightarrow
\bigl(1-e^{-\alpha r}\bigr)+e^{-\alpha r}\bigl[1-(1-\alpha)e^{\alpha}\bigr]
=1-(1-\alpha)\,e^{-\alpha(r-1)}.
\]
Finally, $(1-\alpha)e^{\alpha}$ is strictly decreasing on $(0,1)$ from the value one at $\alpha=0$, so $(1-\alpha)e^{\alpha}<1$ and the portfolio limit strictly exceeds the private limit for every $\alpha\in(0,1)$. \qed

\emph{Symmetric equal-split market.} Write the scaled posterior weight of box $b$ as $W_b=M\pi_b=r^{C_b}/(Z_M/M)$ with $Z_M=\sum_b r^{C_b}$, and recall from \cref{prop:scaling} the notation $f_j=e^{-\alpha}\alpha^j/j!$, $\psi_j=r^je^{-\alpha(r-1)}$, $g_j=f_j\psi_j=e^{-\alpha r}(\alpha r)^j/j!$, and $h(x)=(1-e^{-x})/x$.

\emph{Step 1: weighted occupancy law of large numbers.} For a false box, $\E\bigl[r^{C_b}\bigr]=\bigl(1+(r-1)q_M\bigr)^{N_M}\to e^{\alpha(r-1)}$ and $\E\bigl[r^{2C_b}\bigr]=\bigl(1+(r^2-1)q_M\bigr)^{N_M}\to e^{\alpha(r^2-1)}<\infty$. For two false boxes $a\neq b$, each clue lands on at most one of them, so $\E\bigl[r^{C_a}r^{C_b}\bigr]=\bigl(1+2(r-1)q_M\bigr)^{N_M}$, and a direct expansion shows the covariance is $O(1/M)$ times a bounded factor. Chebyshev's inequality then gives $Z_M/M\to e^{\alpha(r-1)}$ in probability, so $W_b\to\psi_j$ in probability for a box with $C_b=j$; the same second moments give, for the truncated weighted sums, $M^{-1}\sum_b r^{C_b}\1\{C_b\ge J\}\to e^{\alpha(r-1)}\sum_{j\ge J}g_j$ in probability, with the tail uniformly small in $J$ because $\sum_j g_j=1$ (the identity $f_j\psi_j=g_j$ is the exponential tilting of a Poisson law).

\emph{Step 2: convergence of the budget equation.} At the finite anonymous equilibrium of \cref{prop:waterfill}, an active box satisfies $W_b\,\phi(s_b)=\Lambda_M:=M\lambda_M$, and the budget constraint $\sum_b s_b=1$ reads $M^{-1}\sum_b x_b=N_M/M$ with $x_b=N_Ms_b$. Since $(1-x/N)^N\to e^{-x}$ uniformly on compact $x$-sets, $\phi(x/N)\to h(x)$ uniformly on compacts, with $h$ continuous and strictly decreasing from $h(0)=1$. Define the limit budget function
\[
B(\Lambda)=\sum_{j\,:\,\psi_j>\Lambda} f_j\,h^{-1}\!\bigl(\Lambda/\psi_j\bigr).
\]
Because $h(x)\le 1/x$, every active class obeys $x_j\le\psi_j/\Lambda$, so $B(\Lambda)\le\Lambda^{-1}\sum_jg_j=1/\Lambda$ and the series converges locally uniformly; $B$ is therefore continuous, strictly decreasing where positive, with $B(0^+)=\infty$ and $B(\infty)=0$, and \cref{eq:market-limit} has a unique root $\Lambda$. The same bound $x_b\le W_b/\Lambda$ controls the high-count contribution to the finite budget uniformly in $M$ by Step 1, so the finite (random) budget functions converge to $B$ in probability at every fixed $\Lambda$. Strict monotonicity then squeezes the equilibrium value: for any $\Lambda'<\Lambda<\Lambda''$, eventually $B_M(\Lambda')>N_M/M>B_M(\Lambda'')$ with probability tending to one, so $\Lambda_M\to\Lambda$ in probability.

\emph{Step 3: from values to discovery.} The welfare identity of \cref{prop:welfare} gives conditional discovery $N_M\lambda_M=(N_M/M)\Lambda_M\le1$, so $\Lambda_M\le M/N_M$ is uniformly bounded, and bounded convergence yields $G_{\mkt}^{\mathrm{sym}}=\E[N_M\lambda_M]\to\alpha\Lambda$.

\emph{Step 4: the tilted form.} Multiplying the budget equation by $\Lambda$ and using $\Lambda=\psi_jh(x_j)$ on active classes,
\[
\alpha\Lambda
=\sum_j f_jx_j\,\psi_jh(x_j)
=\sum_j f_j\psi_j\bigl(1-e^{-x_j}\bigr)
=\sum_j g_j\bigl(1-e^{-x_j}\bigr):
\]
in the limit the market covers the target's own count class $j$ with probability $1-e^{-x_j}$, where the target's count is $\mathrm{Poisson}(\alpha r)$ and the number of occupants of any fixed box is asymptotically $\mathrm{Poisson}(x_j)$. Canonically $\Lambda=1.094021$ lies strictly between $\psi_1=0.948$ and $\psi_2=3.555$, so the limit market abandons every box with fewer than two reports. \qed

\emph{Abundant evidence (\cref{rem:abundant}).} Hold $p$ fixed. By the multiplicative Chernoff bound, $\Prob(C_\theta\le N_Mp/2)\le e^{-N_Mp/8}\to0$. Each false count is $\mathrm{Binomial}(N_M,q_M)$ with mean $N_Mq_M\to\alpha(1-p)$, so the standard upper-tail bound gives $\Prob(C_b\ge N_Mp/2)\le\bigl(2eq_M/p\bigr)^{N_Mp/2}$, whose base is $O(1/M)$ while the exponent grows linearly in $M$; the union bound over the $M-1$ false boxes therefore vanishes. With probability tending to one the target is the unique plurality winner, so $G_{\cons}\to1$. The same argument applies whenever $\lambda_M=N_Mp_M$ satisfies $\lambda_M/\log M\to\infty$: the target count concentrates above $\lambda_M/2$ while $M\exp\bigl[-(\lambda_M/2)\log(\lambda_M/2e\alpha)\bigr]\to0$. \qed

\section{Computational verification and status of claims}
\label{app:computation}

The accompanying script \texttt{verify\_shared\_discovery\_v8\_6.py}, available with all data tables and figures at \texttt{github.com/yoheinakajima/shared-discovery-paradox}, performs the following checks.

\begin{enumerate}
\item It verifies the reciprocal-share identity and the water-filling equilibrium on $200$ random posteriors, including dominant-posterior pure corners.
\item It verifies the welfare identity $G_{\mkt}^{\mathrm{sym}}=N\lambda$ posterior by posterior.
\item It reproduces the uniform-posterior formula and the $e/(e-1)$ limit.
\item It verifies the tight $2-1/N$ family and the sole-rescue pure-equilibrium characterization.
\item It exactly enumerates the $67$ canonical report-count classes for every copying value, verifies that class probabilities sum to one, and reproduces the private-search closed form.
\item It reproduces the independent pooled frontier, the canonical market value $0.599099252439$, the full copying table, and the crossover $c^*=0.788461656521$.
\item It verifies the blind-control values $8/16=0.50$ and $1-(1-1/16)^8=0.403281$, and that the blind-sentinel hybrid family is maximized at zero sentinels.
\item It computes the average-quality column exactly ($Q_{\mkt}^{\mathrm{sym}}=0.344136160273$) and verifies the redundancy identity $NQ-G=\E[(K-1)^{+}]$ class by class for the market and in closed form for the other protocols.
\item It evaluates the proportional sparse-evidence limits at $\alpha=1/2$, $r=3.75$, checks the fill-identity form of the portfolio limit, and runs a seeded Monte Carlo convergence diagnostic at $M\in\{16,64,256\}$.
\item It cross-checks a closed-form expression for the common-cue likelihood coefficients against the direct enumeration on three $(M,N,p)$ configurations, and validates a batched equilibrium solver against the scalar solver on random posteriors.
\item It solves the market limit equation \cref{eq:market-limit} to high precision ($\Lambda=1.094021073740$, $G=0.547010536870$), verifies the tilted-form identity to $10^{-11}$, confirms that the active support starts at two reports, and runs a seeded Monte Carlo convergence diagnostic for the finite market at $M\in\{16,64,256,1024\}$.
\item It computes the partial-conformity curve of \cref{rem:conformity} exactly for $a=0,\ldots,8$, verifies the endpoints $G(0)=G_{\priv}$ and $G(8)=G_{\cons}$, and confirms the strict monotone decline.
\item It verifies the limit gap decomposition $G_{\priv}-G_{\mkt}^{\mathrm{sym}}\to g_1+\sum_{j\ge2}g_je^{-x_j}=0.287541+0.012094$ to $10^{-11}$, and runs a seeded fixed-$p$ Monte Carlo confirming the abundant-evidence remark (consensus $0.385\to0.947\to1.000$ at $M\in\{16,64,256\}$).
\item It scans the $72$-combination comparative-statics grid, verifies that the market--private gap changes sign at most once in every combination, and writes the crossover locations to \texttt{grid\_crossings.csv}.
\item It generates every figure and the underlying CSV tables.
\end{enumerate}

\paragraph{Analytically established.}
The value-of-information and fixed-budget dominance results; posterior-count ranking under independent symmetric signals; private coverage dominance over one-action consensus; the weak monotone harm of partial conformity by set inclusion; the blind-control values; the redundancy identity $NQ-G=\E[(K-1)^{+}]$; the proportional sparse-evidence limits, including the market cell of \cref{eq:market-limit} and the abundant-evidence consistency remark; the reciprocal-share identity; the exact-potential and payout identities for equal sharing; the anonymous water-filling equilibrium including its pure corner; $G_{\mkt}^{\mathrm{sym}}=N\lambda$; the exact mixed $\PoA=2-1/N$ upper bound and tight family; full pure-strategy implementation under sole-rescue; channel collapse and the closed-form private discovery function; failure of count ranking under dependence by explicit counterexample; the fill identity; and strict monotonicity of the centralized planner gain.

\paragraph{Exactly enumerated or numerically solved for the canonical model.}
Consensus, private, planner, and action-budget frontier values; the partial-conformity curve $G(a)$; recovery budget; expected distinct-action counts; the anonymous market value and its average-quality column; the common-cue table; and the market/private crossover. Enumeration over report classes is exact up to floating-point evaluation of finite formulas; one-dimensional equilibrium and crossover equations are solved to high numerical precision.

\paragraph{Computational patterns, not general theorems.}
The canonical herding discount vanishes for action budgets $L\ge3$ on the tested copying grid; the canonical recovery budget is nonincreasing over that grid; and the market/private gap changes sign at most once on the tested $72$-combination grid---exactly once in $52$ combinations and never in the remaining $20$, where private report-following stays ahead through $c=0.98$. These patterns should not be generalized without additional analysis.

\bibliographystyle{plainnat}
\bibliography{shared_discovery_v8}

\end{document}